\documentclass[letterpaper,11pt]{article}

\def\[#1\]{\begin{align}#1\end{align}}
\def\(#1\){\begin{align*}#1\end{align*}}

\def\argmax{\operatornamewithlimits{arg\,max}}
\def\argmin{\operatornamewithlimits{arg\,min}}

\newcommand{\bprf}{\begin{proof}}
\newcommand{\eprf}{\end{proof}}
\newcommand{\blem}{\begin{lemma}}
\newcommand{\elem}{\end{lemma}}

\newcommand{\eqdef}{\triangleq}
\newcommand{\bP}{\mathbb{P}}
\newcommand{\bE}{\mathbb{E}}
\newcommand{\sF}{\mathcal{F}}

\usepackage{url}
\usepackage{ifthen}
\usepackage{cite}

\usepackage{amssymb,amsfonts,amstext,amsthm,mathrsfs}
\usepackage{graphics,latexsym,epsfig,psfrag,wrapfig,comment,paralist}
\usepackage{xspace}
\usepackage{subcaption,bm,multirow}
\usepackage{booktabs}
\usepackage{mathdots}
\usepackage{bbold}

\usepackage{algorithm, algorithmicx, algpseudocode}
\usepackage{prettyref}
\usepackage{listings}
\usepackage{tikz}
\usetikzlibrary{shapes}
\usetikzlibrary{positioning, arrows, matrix, calc, patterns, fit}
\usepackage{caption}
\usepackage{array}
\usepackage{mdwmath}
\usepackage{multirow}
\usepackage{mdwtab}
\usepackage{eqparbox}
\usepackage{multicol}
\usepackage{amsfonts}
\usepackage{tikz}
\usepackage{multirow,bigstrut,threeparttable}
\usepackage{amsthm}
\usepackage{array}
\usepackage{bbm}
\usepackage{epstopdf}
\usepackage{mdwmath}
\usepackage{mdwtab}
\usepackage{eqparbox}
\usepackage{tikz}
\usepackage{latexsym}
\usepackage{amssymb}
\usepackage{bm}
\usepackage{graphicx}
\usepackage{mathrsfs}
\usepackage{epsfig}
\usepackage{psfrag}
\usepackage{setspace}

\usepackage[cmex10]{amsmath} 
\usepackage{setspace}

\newcommand{\TV}{\mathsf{TV}}
\newcommand{\tTV}{\widetilde{\mathsf{TV}}}
\newcommand{\bR}{\mathbb{R}}

\newcommand{\GG}{\mathcal{G}}
\newcommand{\sH}{\mathcal{H}}

\newtheorem{remark}{\textbf{Remark}}
\newtheorem{theorem}{Theorem}
\newtheorem{corollary}{Corollary}
\newtheorem{lemma}{\textbf{Lemma}}
\usepackage[normalem]{ulem}

\usepackage{color}
\begin{document}

\title{Robust Estimation for Nonparametric Families via Generative Adversarial Networks}

\author{Banghua Zhu, Jiantao Jiao, Michael I. Jordan\thanks{Banghua Zhu is with the Department of Electrical Engineering and Computer Sciences, University of California, Berkeley. Jiantao Jiao is with the Department of Electrical Engineering and Computer Sciences and the Department of Statistics, University of California, Berkeley. Michael I. Jordan is with the Department of Electrical Engineering and Computer Sciences and the Department of Statistics, University of California, Berkeley.  Email: \{banghua, jiantao,michael\_jordan\}@berkeley.edu.}}

\maketitle

\begin{abstract}
We provide a general framework for designing Generative Adversarial Networks (GANs) to solve high-dimensional robust statistics problems, which aim at estimating unknown parameter of the true distribution given adversarially corrupted samples. Prior work~\cite{gao2018robust, gao2019generative} focus on the  problem of robust mean and covariance estimation when the true distribution lies in the family of Gaussian distributions or elliptical distributions, and analyze depth or scoring rule based GAN losses for the problem. Our work extend these to robust mean estimation, second-moment estimation, and robust linear regression when the true distribution only has bounded Orlicz norms, which includes the broad family of sub-Gaussian, sub-exponential and bounded moment distributions. We also provide a different set of sufficient conditions for the GAN loss to work: we only require its induced distance function to be a cumulative density function of some light-tailed distribution, which is easily satisfied by neural networks with sigmoid activation. In terms of techniques, our proposed GAN losses can be viewed as a smoothed and generalized Kolmogorov-Smirnov distance, which overcomes the computational intractability of the original Kolmogorov-Smirnov distance used in the weaken the distance approach in~\cite{zhu2019generalized}. 
\end{abstract}

\section{Introduction}

High-dimensional robust statistics focuses on estimating unknown parameters of high-dimensional distributions given corrupted samples. Formally, assume that the true distribution $p^\star$ with dimension $d$ is altered to some corrupted distribution $p$ such that $\TV(p^\star, p)\leq \epsilon$, where $\TV$ is the total variation distance. One takes $n$ samples from $p$ to form the empirical distribution $\hat p_n$. Given the empirical distribution and a loss function $L$, we aim to output a parameter $\hat \theta(\hat p_n)$ such that $L(p^\star, \hat \theta(\hat p_n))$ is small. 
As a concrete example in the setting of robust mean estimation, one might take the loss function to be $L(p^\star, \hat \theta(\hat p_n)) = \|\mathbb{E}_{p^\star}[X] -\hat \theta(\hat p_n)\|_2$, which is the $\ell_2$ distance between the true mean and an algorithm's output. 

To achieve robustness, it is necessary to impose nontrivial assumptions on the true distribution $p^\star$. Classical asymptotic statistical theory  focused on minimum-distance functionals for  Gaussian or elliptical distributions~\cite{huber1973robust,donoho1988automatic,adrover2002projection,huber2011robust,chen2002influence,  gao2017robust, gao2019generative, zhu2020does}. Recent work has instead focused on computationally efficient algorithms that are based on weaker assumptions regarding the tail behavior of distributions~\cite{diakonikolas2017being, steinhardt2017resilience, steinhardt2017certified, diakonikolas2018learning, diakonikolas2018sever,  liu2018high,  chen2018robust, bateni2019minimax, lecue2019robust, zhu2019deconstructing, zhu2019generalized}.
Another recent line of work is more algorithmic, focusing on generative adversarial networks (GANs) as a promising framework for developing robust estimators. Theoretical work has shown that when the true distribution lies in the family of Gaussian or elliptical distributions, one can design specific GAN losses to achieve theoretical guarantee on robust mean estimation and robust covariance estimation~\cite{gao2018robust, gao2019generative}.

In this work, we take the theoretical study of GANs for robust estimation further, moving beyond strong Gaussian assumptions and providing concrete sufficient conditions for robustness. Specifically, our main contributions are as follows:
\begin{itemize}
    \item We extend the design of GANs for robust estimation from Gaussian and elliptical distribution in~\cite{gao2018robust, gao2019generative} to a rich family of nonparametric classes of distributions, including the family of sub-Gaussian, sub-exponential and bounded moment distributions.
    \item We identify sufficient conditions for robust estimation which pinpoint a broad family of GANs that provably succeed in robust mean estimation, robust covariance estimation and robust linear regression.
    \item  We provide computationally tractable methods for the ``weaken the distance'' approach in~\cite{zhu2019generalized} by smoothing the Kolmogorov-Smirnov distance, which appears to be difficult to optimize. 
\end{itemize}

\section{Preliminaries}

We begin by introducing basic notation.  A function $\psi: [0, +\infty) \mapsto [0, +\infty)$ is called an \emph{Orlicz function} if $\psi$ is convex, non-decreasing, and satisfies $\psi(0) = 0$, $\psi(x) \to \infty$ as $x\to \infty$. For a given Orlicz function $\psi$, the Orlicz norm of a random variable $X$ is defined as
$  \| X \|_{\psi} \triangleq \inf\left \{ t>0 : \bE_p\left [\psi\left(|X|/t\right )\right ] \leq 1\right \}.$ {We use $\|A\|_2 = \sup_{v\neq 0} \|Av\|_2/\|v\|_2$ to denote the spectral norm of the matrix $A$.}
For univariate random variables $X$ and $Y$, we say that
$Y$ \emph{stochastically dominates} $X$ (to first order) 
if $\bP(X \leq t)\geq  \bP(Y \leq t)$ for all $t \in \bR$ \cite{marshall1979inequalities}. We define the generalized inverse of a non-decreasing function $\psi$ as 
$\psi^{-1}(y) = \inf\{x \mid \psi(x) > y\}$. We say a distribution $r$ is an $\epsilon$-deletion of another distribution $p$ if $r\leq \frac{p}{1-\eta}$. The inequality can be formally understood as $\frac{dr}{dp}\leq \frac{1}{1-\eta}$, where $\frac{dr}{dp}$ is the Radon-Nikodym derivative, which can also be understood as $r(A)\leq \frac{p(A)}{1-\epsilon}$ for any set $A$; an equivalent characterization is that
$r$ can be obtained from $p$ by conditioning on an event $E$ of probability $1-\epsilon$.
\subsection{Tasks considered in Robust Estimation}
We focus on three tasks in robust estimation---robust mean estimation, robust second-moment estimation, and robust linear regression---throughout this paper. For robust mean estimation, we take the loss function to be $L(p^\star, \hat \theta(\hat p_n)) = \|\mathbb{E}_{p^\star}[X] -\hat \theta(\hat p_n)\|_2$; for second-moment estimation, we use the loss function  $L(p^\star, \hat \theta(\hat p_n)) = \|\mathbb{E}_{p^\star}[XX^\top] -\hat \theta(\hat p_n)\|_2$, where the algorithm output $\hat\theta(\hat p_n)$ is a matrix; and for linear regression, we take the loss function to be the excess predictive loss, $L(p^\star, \hat \theta(\hat p_n)) = \mathbb{E}_{p^\star}[(Y-X^\top\hat \theta(\hat p_n))^2 -(Y-X^\top\theta^{\star}(p^\star))^2 ]$, where $\theta^\star(p^\star)=\argmin_{\theta}\mathbb{E}_{p^\star}[(Y-X^\top\theta)^2].$

When the output $\hat\theta$ is a parameter of a distribution, the losses can alternatively be written as distances between distributions. For mean estimation, the loss can be written as $L(p^\star, q) = \|\mathbb{E}_{p^\star}[X] -\mathbb{E}_{q}[X]\|_2$; for second-moment estimation,  the loss can be written as $L(p^\star, q) = \|\mathbb{E}_{p^\star}[XX^\top] -\mathbb{E}_{q}[XX^\top] \|_2$; and for linear regression, the loss function can be written as $L(p^\star, q) = \mathbb{E}_{p^\star}[(Y-X^\top \theta^\star(q))^2 -(Y-X^\top\theta^{\star}(p^\star))^2 ]$.

Note that the losses for mean and second-moment estimation admit a pseudonorm representation, $W_\sF(p, q)=\sup_{f\in\sF}(\mathbb{E}_p[f(X)]-\mathbb{E}_q[f(X)])$. In the case of mean estimation, we have $\mathcal{F}_{\mathsf{mean}}=\{f(x)=v^\top x\mid v\in\mathbb{R}^d, \|v\|_2\leq 1\}$, and for second-moment estimation, we have  $\mathcal{F}_{\mathsf{sec}}=\{f(x)=\xi (v^\top x)^2\mid v\in\mathbb{R}^d, \|v\|_2\leq 1, \xi\in\{\pm 1 \}\}$. When we state our main theorem we will assume that the loss admits this pseudonorm representation, and showcase how to deal with linear regression as an exception.

\subsection{Generalized Resilience and Minimum Distance Functional}

It is shown in~\cite[Section 3.2]{zhu2019generalized} that when the true distribution lies in a ``generalized resilience'' family, it is possible to design algorithms with robustness guarantees.  In particular, when the loss function takes the form of a pseudonorm $W_\sF(p, q)$, the generalized resilience set can be expressed simply as: 
\begin{align}
\mathcal{G}_{\mathcal{F}}(\rho) = & \{p \mid \bE_{r}[f(X)] - \bE_{p}[f(X)] \leq \rho(\epsilon), \forall \epsilon\in[0, 1), \nonumber \\ 
& \quad \forall r\leq \frac{p}{1-\epsilon}, f\in\mathcal{F}\}.\label{eqn.def_G_TV_WF}
\end{align}

For convenience of notation, we use $\mathcal{G}_{\mathsf{mean}}$ to denote $\mathcal{G}_{\mathcal{F}_\mathsf{mean}}$, and $\mathcal{G}_{\mathsf{sec}}$ to denote $\mathcal{G}_{\mathcal{F}_\mathsf{sec}}$. 
It is shown in~\cite{zhu2019generalized}  that $\mathcal{G}_\mathsf{mean}(\epsilon\psi^{-1}(\epsilon))$ is a superset of the set of bounded Orlicz norm distributions, $\{X\sim p\mid \sup_{\|v\|_2\leq 1}\|v^\top X\|_\psi\leq 1\}$, which implies a bounded $k$-th moment  when $\psi(x)=x^k$, sub-exponentiality when $\psi(x)=\exp(x)$, and sub-Gaussianity when $\psi=\exp(x^2)$. Similarly, $\mathcal{G}_\mathsf{sec}(\epsilon\psi^{-1}(\epsilon))$ is a superset of $\{X\sim p\mid \sup_{\|v\|_2\leq 1}\|(v^\top X)^2\|_\psi\leq 1\}$.

The generalized resilience set takes a slightly more complicated form for linear regression. However, it is also a superset of some bounded Orlicz norm family. We use $\mathcal{G}_{\mathsf{reg}}(\psi)$ to denote the following family:
 \begin{align} 
  \GG_{\mathsf{reg}}(\psi) = \Bigg\{p \mid &\bE_{p}\bigg[\psi \bigg(\frac{(v^{\top}X)^2}{\sigma_1^2 \bE_{p}[(v^{\top}X)^2]}\bigg)\bigg] \leq 1 \text{ for all } v \in \bR^d, \nonumber \\ 
  & \text{ and }\bE_{p}\left[\psi \left(\frac{(Y - X^\top \theta^*(p))^2}{\sigma_2^2}\right)\right] \leq 1 \Bigg\}.\label{eqn.G_linreg}
    \end{align}

Working under the restrictive assumption that the true distribution lies in $\mathcal{G}$, \cite{zhu2019generalized} study a projection algorithm and show that it achieves polynomial sample complexity for robust estimation. The algorithm  projects the  corrupted empirical distribution $\hat p_n$ onto $\mathcal{G}$ under the \emph{generalized Kolmogorov-Smirnov distance}  $\TV_\mathcal{H}$, defined as 
\begin{equation}
\TV_{\sH}(p, q) \eqdef \sup_{f \in \sH, t \in \bR} |\bP_p(f(X) \geq t) - \bP_q(f(X) \geq t)|.
\end{equation}

The distance $\TV_\mathcal{H}$ is smaller than the total variation distance $\TV$ because it takes a supremum over only the events defined by threshold 
functions in $\sH$, while $\TV$ takes the same supremum over all measurable events. The projection algorithm is analyzed for Guassian mean estimation in~\cite{donoho1988pathologies}, with $\mathcal{H} = \{ v^{\top}x \mid v\in\bR^d\}$, and extended to bounded Orlicz norm family for robust second-moment estimation, robust linear regression in~\cite{zhu2019generalized}.
Unfortunately, however, it is difficult to compute the $\TV_{\mathcal{H}}$ projection in practice due to the lack of differentiability of commonly used losses. We show how to circumvent this issue in the following section.

\section{A Unified Framework for Design of GANs}

In this section we show that a large family of smoothed loss function can be used as projection functionals and achieve favorable finite sample bounds. We define the 
\emph{smoothed generalized Kolmogorov-Smirnov distance}  as
\begin{align}\label{eqn.def_smoothed_KS}
    \tTV_\mathcal{F}(p, q) =  \sup_{ f\in\sF, t\in\bR} |\bE_p[T(f(X)+t)] - \bE_q[T(f(X)+t)]|,
\end{align}
where we assume that $T(x)$, after an affine transformation $aT(x)+b$, can be written as the cumulative distribution function (CDF) of some random variable $Z$; i.e., $aT(-\infty)+b = 0, aT(+\infty)+b=1 $ and $T$ is right-continuous.  An example is the sigmoid function $T(x) = 1/(1+e^{-x})$. 
When $Z = 0$ almost surely, the smoothed distance $\tTV_\sF$ reduces to the case of generalized KS distance. The key observation is there is a generalization of the \emph{mean cross} lemma in~\cite[Lemma 3.3]{zhu2019generalized} that works for the 
\emph{smoothed generalized Kolmogorov-Smirnov distance}.
\begin{lemma}[Closeness in smoothed KS distance implies mean cross]\label{lem.smoothed_KS}
Assume for two distribution $p, q$, 
\begin{align}
    \sup_{t\in\bR} |\bE_p[T(X+t)] - \bE_q[T(X+t)] |\leq \epsilon,
\end{align}
where we assume that there exist $a, b\in\mathbb{R}$ such that $aT(x)+b$ can be written as the CDF of some random variable $Z$. Assume the distribution of $Z$ is inside the resilience family of $\GG_{\mathsf{mean}}(\rho_Z)$. Then there exist $r_p \leq \frac{p}{1-|a|\epsilon}$ and $r_q\leq \frac{q}{1-|a|\epsilon}$ such that 
\begin{align}
    \bE_{r_q}[X] -  \bE_{r_p}[X] \leq 2\rho_Z(|a|\epsilon). 
\end{align}
\end{lemma}
For the generalized KS distance $\TV_\sF$, $Z \equiv 0$ almost surely, and $\rho_Z = 0$, which corresponds to Lemma~\cite[Lemma 3.3]{zhu2019generalized}. 
\begin{proof}[Proof of Lemma~\ref{lem.smoothed_KS}]
For a fixed $x$,  $aT(x+t)+b = \bP(Z\leq x + t) = \bP(x - Z \geq -t)$. Thus we know
\begin{align}\label{eqn.lemma_close}
   \sup_{ t\in\bR} | \bP_q(X-Z\geq -t) - \bP_p(X-Z\geq -t)|  \leq |a|\epsilon,
\end{align}
where $X$ and $Z$ are independent. Denote $\tilde p$ as the distribution of $X -Z$ when $X\sim p$, and $\tilde q$ as the distribution of $X -Z$ when $X\sim q$.  Starting from $\tilde p, \tilde q$, we delete  probability mass of size $\epsilon$ corresponding to the largest points in $\tilde p$ to get $\tilde r_p$, and delete probability mass of size $\epsilon$  corresponding to the smallest points  $\tilde q$ to get $\tilde r_q$. Equation (\ref{eqn.lemma_close}) implies that $\bP_{Y\sim\tilde r_p}(Y\geq t) \leq \bP_{Y\sim\tilde r_q}(Y\geq t)$ holds for all $t\in\mathbb{R}$. Hence, $\tilde r_q$ stochastically dominates $r_p$ and $\bE_{\tilde r_p}[Y] \leq \bE_{\tilde r_q}[Y]$. Thus  we know that 
there exist $\tilde r_p\leq\frac{\tilde p}{1-|a|\epsilon}$ and $\tilde r_q \leq \frac{\tilde q}{1-|a|\epsilon}$ such that
\begin{align*}
     \bE_{\tilde r_q}[X-Z]  &\leq  \bE_{\tilde r_p}[X-Z].
\end{align*}
Denote the original distribution of $Z$ as $p_Z$. Note that the deletion process that yields $\tilde r_p$ and 
$\tilde r_q$ can be viewed as operating on the joint distribution of $X, Z$, such that the marginal distributions of $\tilde r_p$ and $\tilde r_q$ are obtained as deletions of $p_Z$. Rearranging, we obtain
\begin{align}
    \bE_{\tilde r_q}[X] -  \bE_{\tilde r_p}[X] & \leq  \bE_{\tilde r_q}[Z] - \bE_{\tilde r_p}[Z] \nonumber \\
    & \leq |\bE_{\tilde r_q}[Z] - \bE_{p_Z}[Z]| +  |\bE_{\tilde r_p}[Z] - \bE_{p_Z}[Z]|. 
    \nonumber \\
    & \leq 2\rho_Z(|a|\epsilon).\nonumber 
\end{align}
\end{proof}

We now show that given the mean cross lemma for smoothed KS distance, the smoothed distance can  be used to deliver a small statistical error for robust estimation under a pseudonorm-based loss function $W_\sF$. We will show later how the analysis extends beyond pseudonorm loss functions by considering the special case of linear regression.

\begin{theorem}\label{theorem.GAN}
For any $p^* \in \GG_{\sF}(\rho)$ in Equation (\ref{eqn.def_G_TV_WF}), denote by $p$ an observed corrupted distribution such that $\TV(p^*,p)\leq \epsilon$. Define 
\begin{align}
    A(p, q) =  \sup_{(d_1, d_2)\in\mathcal{D}} \bE_p[d_1(X)] + \bE_q[d_2(X)].
\end{align}
Here $\mathcal{D}$ is some family of discriminator function pairs $(d_1, d_2)$. For some function $T(x)$, let $aT(x)+b$ be the CDF of some random variable $Z$. 
Assume the following conditions:
\begin{enumerate}
\item For all $(d_1, d_2)\in \mathcal{D}$, $x\in \bR$, we have  $|d_2(x)|\leq 1/2$. 
\item For any distribution pair $p, q\in \GG_{W_\sF}(\rho)$, we have
\begin{align}
& A(q, p) - A(p, p) \leq  \epsilon \Rightarrow \label{eqn.gan_condition_implication}\\ 
&  \sup_{f\in\sF, t\in\bR} | \bE_q\big[T(f(X)+t)\big] - \bE_p\big[T(f(X)+t)\big] |  \leq C \epsilon \nonumber 
\end{align}
for some constant $C>0$ and any $\epsilon<1/C|a|$. %$\tilde \epsilon = 2\epsilon + 2\bar{ A}(p, \hat p_n)$, where
%\begin{align}
 %   \bar{A}(p, q) = \sup_{ (d_1, d_2) \in\mathcal{D}} |\bE_{p}[d_2(X)] - \bE_{q}[d_2(X)]|.
%\end{align}
\item The distribution of $Z$ is inside the resilient set $ \GG_{\mathsf{mean}}(\rho_Z)$, 
\end{enumerate}
Then the projection algorithm $q = \argmin_{q\in\GG_{{\mathcal{F}}}(\rho)} A(q, \hat p_n)$  guarantees
\begin{align}
    W_\sF(p^*, q) \leq 2\rho({C\tilde \epsilon}) +2\rho_Z(C\tilde \epsilon),
\end{align}
where 
$\tilde \epsilon = 2\epsilon + 2\bar{ A}(p, \hat p_n)$, and
\begin{align}
   \bar{A}(p, q) = \sup_{ (d_1, d_2) \in\mathcal{D}} |\bE_{p}[d_2(X)] - \bE_{q}[d_2(X)]|.
\end{align}
\end{theorem}
\begin{remark}
The three conditions are weak enough to be satisfied for a large family of neural networks. 
The first condition on the magnitude of $d_2$ can be easily satisfied by setting the output of the neural network designed for $d_2$ to pass through a bounded activation function. The second condition  is satisfied when $d_1$ simply takes the form of $T$. The third condition requires a careful design of $T$, and thus $d_1$, which can also be easily satisfied when the network has a bounded activation function. For example, a simple sigmoid or tanh function gives a CDF whose induced distribution is sub-exponential, and a ramp function gives a CDF whose induced distribution is sub-Gaussian. In~\cite{gao2019generative}, it was shown that under mild conditions, one can produce valid $(d_1, d_2)$ using proper scoring rules and appropriate neural network architectures to ensure the three conditions. Our result, combined with~\cite{gao2019generative}, extends the results for mean and second-moment estimation in~\cite{gao2019generative} to general resilient sets including sub-Gaussian and bounded $k$-th moments, while \cite{gao2019generative} gives guarantee for elliptical distributions as semi-parametric classes.
\end{remark}

\begin{proof}[Proof of Theorem~\ref{theorem.GAN}]

The proof mainly focuses on verifying two important properties in~\cite{zhu2019generalized}.
\begin{enumerate}
    \item \textbf{Robustness to perturbation: }For any distribution $p_1, p_2, p_3$, we have
\begin{align}
     \lefteqn{|A(p_1, p_2) - A(p_1, p_3)|} \nonumber \\ 
     &= |\sup_{(d_1, d_2) \in\mathcal{D}} (\bE_{p_1}[d_1(X)] + \bE_{p_2}[d_2(X)]) \nonumber \\ 
     & \qquad - \sup_{(d_1, d_2) \in\mathcal{D}} (\bE_{p_1}[d_1(X)] + \bE_{p_3}[d_2(X)]) |\nonumber \\
     & \leq \sup_{(d_1, d_2) \in\mathcal{D}}|\bE_{p_1}[d_1(X)] + \bE_{p_2}[d_2(X)]  \nonumber \\
     & \qquad - (\bE_{p_1}[d_1(X)] + \bE_{p_3}[d_2(X)])| \nonumber \\
     & \leq \sup_{(d_1, d_2) \in\mathcal{D}} |\bE_{p_2}[d_2(X)] - \bE_{p_3}[d_2(X)]|\nonumber \\
     & =  \bar{A}(p_2,p_3) \nonumber \\
     & \leq  \sup_{\sup_{x}|d_2(x)|\leq \frac{1}{2}} |\bE_{p_2}[d_2(X)] - \bE_{p_3}[d_2(X)]|\nonumber \\
     & = \TV(p_2, p_3).\label{eqn.robust_perturbation}
\end{align}
Combining Equation~\eqref{eqn.robust_perturbation}, the fact $q = \argmin_{q\in\mathcal{G}_\mathcal{F}} A(q, \hat p_n)$ and  $p^\star\in \mathcal{G}_\mathcal{F}$, one has 
\begin{align*}
    &A(q, p^\star)-A(p^\star, p^\star) \\
    & \leq A(q, p)+\epsilon-A(p^\star, p^\star) \\ 
    & \leq A(q, \hat p_n)+\epsilon+\bar A(p, \hat p_n) -A(p^\star, p^\star)\\ 
    & \leq A(p^\star, p_n)+\epsilon+\bar A(p, \hat p_n) -A(p^\star, p^\star) \\ 
    & \leq A(p^\star, p)+\epsilon+2\bar A(p, \hat p_n) -A(p^\star, p^\star) \\ 
    & \leq 2\epsilon+2\bar A(p, \hat p_n) = \tilde \epsilon.
\end{align*}

\item \textbf{Generalized modulus of continuity: }
For any $p, q \in \GG_{W_\sF}$, from $A(q, p) - A(p, p) \leq \tilde \epsilon$ and condition 2) in Equation~\eqref{eqn.gan_condition_implication}, we have
\begin{align*}
\sup_{f\in\sF, t\in\bR} | \bE_q\big[T(f(X)+t)\big] - \bE_{p^\star}\big[T(f(X)+t)\big]  |  \leq C\tilde \epsilon.
\end{align*}

From Lemma~\ref{lem.smoothed_KS}, we know that for any fixed $f^* \in \argmax_{f\in\sF} \bE_{p^\star}[f(X)] - \bE_q[f(X)]$, there exist $r_{p^\star}\leq \frac{{p^\star}}{1-C|a|\tilde \epsilon}, r_q\leq \frac{q}{1-C|a|\tilde \epsilon}$, such that 
\begin{align*}
    \bE_{\tilde r_{p^\star}}f^*(X) \leq  \bE_{\tilde r_q}f^*(X) + 2\rho_Z(C|a| \tilde \epsilon).
\end{align*}

From $q\in \GG_{\sF}(\rho)$, we have
\begin{align}
\forall r \leq \frac{q}{1-C|a|\tilde \epsilon},
    W_\sF(r, q) \leq \rho(C|a|\tilde \epsilon).
\end{align}
Therefore
\begin{align}
W_\sF({p^\star}, q) & = \bE_{p^\star}[f^*(X)] - \bE_q[f^*(X)] \nonumber \\
& \leq \bE_{p^\star}[f^*(X)] - \bE_{\tilde r_q}[f^*(X)] + \rho(C|a|\tilde \epsilon) \nonumber \\
& \leq \bE_{p^\star}[f^*(X)] - \bE_{\tilde r_{p^\star}}[f^*(X)] +2\rho_Z(C|a|\tilde\epsilon)\nonumber  \\ 
& \quad + \rho(C|a|\tilde\epsilon) \nonumber \\
& \leq 2\rho(C|a|\tilde\epsilon) +2\rho_Z(C|a|\tilde\epsilon),
\end{align}
which finishes the proof.
\end{enumerate}
\end{proof}
\section{Application of Theorem~\ref{theorem.GAN} for Robust Estimation}

Now we demonstrate how Theorem~\ref{theorem.GAN} leads to concrete designs of GANs for robust mean estimation, second-moment estimation and robust linear regression.

\subsection{Mean estimation}
We start with the problem of mean estimation. Recall the choice of function set $\mathcal{F}_{\mathsf{mean}} = \{f(x)= v^\top X|v\in\mathbb{R}^d,\|v\|_2\leq 1\}$.
We show that different choices of $A(p, q)$ give similar performance guarantees in terms of robustness. For the convenience of notation, let 
\begin{align*}
    g_1(v, t, X) &= \mathsf{sigmoid}(v^\top X+t), \\
    g_2(w,\{v_j\}, \{t_j\}, X) 
    &= \mathsf{sigmoid}\left(\sum_{j\leq l} w_jg_1(v_j, t_j, X)\right) \\  &= \mathsf{sigmoid}\left(\sum_{j\leq l} w_j\mathsf{sigmoid}(v_j^\top X+t)\right)
\end{align*} which represent one- and two-layer neural networks, respectively, with sigmoid activation functions.  Here $v_j, v$ are $d$-dimensional vectors, $t, t_j\in\mathbb{R}$, and $w$ is a $l$-dimensional vector. We consider the following design of distances $A(p, q)$:
\begin{align}\label{eqn.A_1}
    A_1(p, q) =& \sup_{\|v\|_2\leq 1, t\in\mathbb{R}} |\mathbb{E}_{p}[g_1(v, t, X)] -  \mathbb{E}_{q}[g_1(v, t, X)]|,\\
    A_2(p, q) = &\sup_{\|w\|_1\leq 1, \|v_j\|_2\leq 1, t_j\in\mathbb{R}} |\mathbb{E}_{p}\left[ g_2(w,\{v_j\}, \{t_j\}, X)\right] \nonumber \\ 
    &-  \mathbb{E}_{q}\left[ g_2(w,\{v_j\}, \{t_j\}, X)\right]|,\label{eqn.A_2}\\
    A_3(p, q) = &\sup_{\|w\|_1\leq 1, \|v_j\|_2\leq 1, t_j\in\mathbb{R}} \mathbb{E}_{p}\left[\log\left( g_2(w,\{v_j\}, \{t_j\}, X)\right)\right] \nonumber \\ 
    &+  \mathbb{E}_{q}\left[\log\left(1- g_2(w,\{v_j\}, \{t_j\}, X)\right)\right],\label{eqn.A_3}
\end{align}
where $A_1$ can be viewed as  a simple one-layer discriminator, $A_2$ is a two-layer discriminator, and $A_3$ is a two-layer discriminator with loss chosen as log score, which is analyzed for Gaussian and elliptical distributions in~\cite{gao2019generative}. 
\begin{corollary}\label{cor.mean}
Assume the true distribution $p^\star\in\mathcal{G}_{\mathsf{mean}}(\rho)$.
Let the projection algorithm $q_1 = \argmin_{q\in\mathcal{G}_{\mathsf{mean}}} A_1(q, \hat p_n)$, and $q_2, q_3$ be the projection algorithm that projects under distance $A_2, A_3$. We have
with probability at least $1-\delta$,
\begin{align*}
   \| \mathbb{E}_{p^\star}[X] -  \mathbb{E}_{q}[X]\|_2 \leq C_1\cdot(\rho(\tilde \epsilon) +  \tilde\epsilon\log(1/\tilde \epsilon)),
\end{align*}
where $\tilde \epsilon=C_2(\epsilon+\sqrt{d/n}+\sqrt{\log(1/\delta)/n})$.
\end{corollary}

\begin{remark}
Corollary~\ref{cor.mean} recovers~\cite[Theorem 3.2]{zhu2019generalized} up to an additive  term $\tilde\epsilon\log(1/\tilde\epsilon)$.
As a direct result of Corollary~\ref{cor.mean}, one can see that if $p^\star$ is inside the sub-Gaussian family, the projection algorithm guarantees a rate of $C\cdot \tilde\epsilon \log(1/\tilde \epsilon)$, which is nearly optimal up to logarithmic factors compared to the result for Gaussian~\cite{gao2019generative}. When $p^\star$ has bounded covariance, the rate becomes $(\epsilon+\sqrt{d/n}+\sqrt{\log(1/\delta)/n})^{1/2}$, which is optimal in terms of the dependence on $\epsilon$, but sub-optimal in dependence with respect to $d,n$. 

Compared to the interpretation of proper scoring rules and appropriate network structures in~\cite{gao2019generative}, our analysis provides a different view in which robustness is achieved by controlling the  tail of the induced distribution of $T(x)$. The results in Corollary~\ref{cor.mean} also generalize to deeper neural networks  with bounded outputs where the three assumptions are more readily satisfied.

Although projection to $\mathcal{G}$ is not exactly computable for general $\mathcal{G}$, in practice one may use a subset of Orlicz-norm bounded distributions which can be parameterized using neural networks, or turn the constrained optimization problem into a regularized optimization problem by adding the Orlicz norm as a regularizer in the loss function. 

\end{remark}

\begin{proof}[Proof of Corollary~\ref{cor.mean}]
Since scaling the distance does not change the final projection $q$, 
it suffices to verify the three assumptions in Theorem~\ref{theorem.GAN} for scaled distances $A_1/2, A_2/2, A_3/2$. For the first assumption, one can see that the three distances all satisfy $|d_2(x)|\leq 1/2$ due to the fact that $|g_1|\leq 1, g_2\in [1/2, e/(e+1)]$, $\log(1-g_2)\in[-0.7,-0.3]$.

Now we verify the second assumption. Taking $w_1=1, w_j=0, \forall j>1$ in $g_2$ and applying triangle inequality, we have
\begin{align*}
   & A_1(q, p) -A_1(p, p)\leq \epsilon \Rightarrow\\
   & \sup_{\|v\|_2\leq 1, t\in\mathbb{R}}| \mathbb{E}_{p}[g_1(v, t, X)] -  \mathbb{E}_{q}[g_1(v, t, X)]|\leq \epsilon, \\
  & A_2(q, p) -A_2(p, p)\leq \epsilon    \Rightarrow \\ 
 &  \sup_{\|v\|_2\leq 1, t\in\mathbb{R}}| \mathbb{E}_{p}[\mathsf{sigmoid}(g_1(v, t, X))]- \\ 
 & \quad \mathbb{E}_{q}[\mathsf{sigmoid}(g_1(v, t, X))]|\leq \epsilon, \\
   & A_3(q, p) -A_3(p, p)\leq \epsilon    \Rightarrow \\ 
  & \sup_{\|v\|_2\leq 1, t\in\mathbb{R}}| \mathbb{E}_{p}[\log(\mathsf{sigmoid}(g_1(v, t, X)))]- \\
  &\quad \mathbb{E}_{q}[\log(\mathsf{sigmoid}(g_1(v, t, X)))]|\leq \epsilon.
\end{align*}
These imply that $A_1, A_2, A_3$ satisfy assumption 2) in Theorem~\ref{theorem.GAN} with particular choices of $T$: $T_1(x)=\mathsf{sigmoid}(x), T_2=\mathsf{sigmoid}(T_1(x)), T_3=\log(T_2(x))$.

We now verify the third assumption. For CDF $T_1(x) =  \frac{1}{1+\exp(-x)}$, the corresponding random variable $Z$ is sub-exponential since $\max(\mathbb{P}(X\leq -t),\mathbb{P}(X\geq t))\leq \exp(-t)/2$ for any $t\geq 0$. For $T_2, T_3$, we can verify that after scaling the CDF, $\mathbb{E}[\exp(|Z|/10)]\leq 1$. This shows  the induced random variables from $T_2, T_3$ are also sub-exponential. From~\cite{zhu2019generalized} we know   the distribution of $Z$ lies in the set of $\mathcal{G}_{\mathsf{mean}}(\epsilon\log(1/\epsilon))$. 

Furthermore, from~\cite[Lemma 7.3]{gao2019generative} and~\cite[Lemma 8.1]{gao2018robust}, we know that $\bar A_1(\hat p_n, p), \bar A_2(\hat p_n, p), \bar A_3(\hat p_n, p)$ are all bounded by $C(\sqrt{d/n}+\sqrt{\log(1/\delta)/n})$ with probability at least $1-\delta$.
\end{proof}

\subsection{Second-moment estimation}
Using the same technique as mean estimation, we can show similar results for second-moment estimation by changing $\mathcal{F}_{\mathsf{mean}}$ to $\mathcal{F}_{\mathsf{sec}}$ in $\mathcal{\tTV}_\sF$. Let $A_1, A_2, A_3$ be the same as mean estimation in Equation~\eqref{eqn.A_1},~\eqref{eqn.A_2},~\eqref{eqn.A_3} except that we set
\begin{align*}
    &g_1(v, t, X) = \mathsf{sigmoid}((v^\top X)^2+t), \\
    &g_2(w,\{v_j\}, \{t_j\}, X) 
    = \mathsf{sigmoid}\left(\sum_{j\leq l} w_jg_1(v_j, t_j, X)\right). \\ &\quad = \mathsf{sigmoid}\left(\sum_{j\leq l} w_j\mathsf{sigmoid}((v_j^\top X)^2+t)\right).
\end{align*}
Following the same proof as Corollary~\ref{cor.mean}, we have\footnote{Besides a different choice of $g_1, g_2, \mathcal{F}$, the only difference in the proof is the concentration of $\bar A(\hat p_n, p)$. One can follow a similar analysis in~\cite[Lemma 7.3]{gao2019generative} and~\cite[Lemma 8.1]{gao2018robust} to derive the concentration.}
\begin{corollary}
Assume the true distribution $p^\star\in\mathcal{G}_{\mathsf{sec}}(\rho)$.
Let the projection algorithm $q_1 = \argmin_{q\in\mathcal{G}_{\mathsf{sec}}} A_1(q, \hat p_n)$, and $q_2, q_3$ be the projection algorithm that projects under distance $A_2, A_3$. We have
with probability at least $1-\delta$,
\begin{align*}
   \| \mathbb{E}_{p^\star}[XX^\top] -  \mathbb{E}_{q}[XX^\top]\|_2 \leq C_1\cdot(\rho(\tilde \epsilon) +  \tilde\epsilon\log(1/\tilde \epsilon)),
\end{align*}
where $\tilde \epsilon=C_2(\epsilon+\sqrt{d/n}+\sqrt{\log(1/\delta)/n})$.
\end{corollary}
\begin{remark}
When the true distribution lies in the set of sub-Gaussian distributions, the projection algorithm guarantees a rate of $C\cdot \tilde\epsilon \log(1/\tilde \epsilon)$, which is near optimal up to logarithmic factors compared to the result for Gaussian~\cite{gao2019generative}.
\end{remark}
\subsection{Linear Regression}

For linear regression, we adopt a difference family of $\mathcal{F}$ in $\tTV_\mathcal{F}$, which takes the form of $\{(Y-v_1^\top X)^2-(Y-v_2^\top X)^2\mid v_1,v_2\in\mathbb{R}^d\}$. Let $A_1, A_2, A_3$ be the same as mean estimation in Equation~\eqref{eqn.A_1},~\eqref{eqn.A_2},~\eqref{eqn.A_3} except that we set
\begin{align*}
 &   g_1(v_1, v_2, t, X) = \mathsf{sigmoid}((Y-v_1^\top X)^2-(Y-v_2^\top X)^2+t), \\
    &g_2(w,\{v_j^{(1)}\},\{v_j^{(2)}\}, \{t_j\}, X) \\
    &= \mathsf{sigmoid}\left(\sum_{j\leq l} w_jg_1(v_j^{(1)}, v_j^{(2)}, t_j, X)\right).
\end{align*}

Although the loss for linear regression is not  pseudonorm, and thus Theorem~\ref{theorem.GAN} is not directly applicable,  we show below that under slight modification of the analysis in Theorem~\ref{theorem.GAN} that invokes the general property of generalized resilience in~\cite{zhu2019generalized},  the projection algorithm under the distances $A_1, A_2, A_3$ still  guarantees robust regression. 
\begin{corollary}\label{cor.reg}
Assume the true distribution $p^\star\in\mathcal{G}_{\mathsf{reg}}(\psi)$  in Equation~\eqref{eqn.G_linreg}.
Let the projection algorithm $q_1 = \argmin_{q\in\mathcal{G}_{\mathsf{reg}}} A_1(q, \hat p_n)$, and $q_2, q_3$ be the projection algorithm that projects under distance $A_2, A_3$. We have
with probability at least $1-\delta$,
\begin{align*}
 & \mathbb{E}_{p^\star}[( \theta^\star(q)^\top X-Y)^2 -(\theta^\star(p)^{\top} X-Y)^2 ]  \\ 
  \leq &C_1\cdot(\tilde \epsilon\psi^{-1}(1/\tilde\epsilon) +  \tilde\epsilon\log(1/\tilde \epsilon)),
\end{align*}
where $\tilde \epsilon=C_2(\epsilon+\sqrt{d/n}+\sqrt{\log(1/\delta)/n})$.
\end{corollary}
\begin{remark}
Corollary~\ref{cor.reg} recovers~\cite[Theorem 3.3]{zhu2019generalized} up to an additive  term $\tilde\epsilon\log(1/\tilde\epsilon)$. When $X$ and $Y-X^\top\theta^\star(p^\star)$ are both sub-Gaussian, our dependence on $\epsilon$ is the same as the optimal rate in  Gaussian example in~\cite{gao2017robust} up to a log factor while the sample complexity matches~\cite{gao2017robust} exactly. 
\cite{diakonikolas2019efficient} guarantee an optimal parameter error $\|\hat \theta - \theta^*\|_2 \lesssim (\epsilon\log(1/\epsilon))^2$ given $O(d/\epsilon^2)$ samples when $X$ is isotropic Gaussian and $Z$ has bounded second moment, which is implied by our analysis by taking the Orlicz function $\psi$  as exponential function and use it as the generalized resilience set.
\end{remark}
\begin{proof}[Proof of Corollary~\ref{cor.reg}]
We slightly modify the last part of the proof in Theorem~\ref{theorem.GAN} for linear regression. 
Following the same proof as Theorem~\ref{theorem.GAN}, we know  for fixed $f = (Y-X^\top\theta^\star(q))^2 - (Y-X^\top\theta^\star(p^\star))^2$,  one can find deletions $r_{p^\star}, r_q$ which delete at most $\tilde\epsilon$-fraction of the distributions ${p^\star}, q$, such that
\begin{align}
\lefteqn{\mathbb{E}_{r_{p^\star}}[(Y-X^\top\theta^\star(q))^2 - (Y-X^\top\theta^\star(p^\star))^2]}\nonumber\\ 
&\leq \mathbb{E}_{r_{q}}[(Y-X^\top\theta^\star(q))^2 - (Y-X^\top\theta^\star(p^\star))^2]+2\rho_Z(C\tilde \epsilon)\nonumber  \\ 
&\leq  \mathbb{E}_{r_{q}}[(Y-X^\top\theta^\star(q))^2 - (Y-X^\top\theta^\star(r_q))^2] + 2\rho_Z(C\tilde \epsilon)\nonumber \\ 
&\leq  C\tilde \epsilon\psi^{-1}(1/\tilde\epsilon) + 2\rho_Z(C\tilde \epsilon).\label{eqn.G_up_reg}
\end{align}
Here the last inequality comes from~\cite[Lemma F.3]{zhu2019generalized}. Invoking ~\cite[Lemma F.3]{zhu2019generalized} again, we know that Equation~\eqref{eqn.G_up_reg} implies that 
\begin{align*}
    \lefteqn{\mathbb{E}_{{p^\star}}[(Y-X^\top\theta^\star(q))^2 - (Y-X^\top\theta^\star(p^\star))^2]} \\ 
    &\leq C_1(\tilde \epsilon\psi^{-1}(1/\tilde\epsilon) + \rho_Z(C_2\tilde \epsilon)).
\end{align*}
The rest of the proof follows directly from that of Corollary~\ref{cor.mean} by verifying the three assumptions in Theorem~\ref{theorem.GAN}.
\end{proof}

\newpage
%%%%%%
%% Appendix:
%% If needed a single appendix is created by
%%
%\appendix
%%
%% If several appendices are needed, then the command
%%
% \appendices
%%
%% in combination with further \section-commands can be used.
%%%%%%

%\section*{Acknowledgment}

%We are indebted to Michael Shell for maintaining and improving
% \texttt{IEEEtran.cls}. 

%%%%%%
%% To balance the columns at the last page of the paper use this
%% command:
%%
%\enlargethispage{-1.2cm} 
%%
%% If the balancing should occur in the middle of the references, use
%% the following trigger:
%%
%%
%% which triggers a \newpage (i.e., new column) just before the given
%% reference number. Note that you need to adapt this if you modify
%% the paper.  The "triggered" command can be changed if desired:
%%
%\IEEEtriggercmd{\enlargethispage{-20cm}}
%%
%%%%%%

%%%%%%
%% References:
%% We recommend the usage of BibTeX:
%%
\bibliographystyle{plain}
\bibliography{di}
%%
%% where we here have assume the existence of the files
%% definitions.bib and bibliofile.bib.
%% BibTeX documentation can be obtained at:
%% http://www.ctan.org/tex-archive/biblio/bibtex/contrib/doc/
%%%%%%

\end{document}